%% file: esper.tex
\documentclass{article}

    \PassOptionsToPackage{numbers, compress}{natbib}
\usepackage[ruled,vlined,algo2e]{algorithm2e}

    \usepackage[final]{neurips_2022}

\usepackage[utf8]{inputenc} % allow utf-8 input
\usepackage[T1]{fontenc}    % use 8-bit T1 fonts
\usepackage{hyperref}       % hyperlinks
\usepackage{url}            % simple URL typesetting
\usepackage{booktabs}       % professional-quality tables
\usepackage{amsfonts}       % blackboard math symbols
\usepackage{nicefrac}       % compact symbols for 1/2, etc.
\usepackage{microtype}      % microtypography
\usepackage[table,x11names]{xcolor}         % colors
\usepackage{tikz}
\usepackage{color}
\usepackage{longtable}

\usepackage{listings}
\usepackage{amsthm}
\theoremstyle{definition}
\newtheorem{definition}{Definition}[section]

\input{_macros}
\input{macros-admin}

\newcommand{\describe}[3][0pt]{\hspace*{.12em}\underbracket[0.5pt][1pt]{#2\hspace*{#1}}_\text{#3}}
\newcommand{\algoshort}{ESPER}

\newtheorem{theorem}{Theorem}[section]

\title{You Can't Count on Luck: \\Why Decision Transformers and RvS\\ Fail in Stochastic Environments}

\author{Keiran Paster\\
Department of Computer Science\\
University of Toronto, Vector Institute\\
\texttt{keirp@cs.toronto.edu} \\
\And
Sheila A. McIlraith \& Jimmy Ba \\
Department of Computer Science \\
University of Toronto, Vector Institute\\
\texttt{\{sheila, jba\}@cs.toronto.edu} \\
}

\begin{document}

\maketitle

\begin{abstract}
Recently, methods such as Decision Transformer \citep{chen_decision_2021} that reduce reinforcement learning to a prediction task and solve it via supervised learning (RvS) \citep{emmons_rvs_2021} have become popular due to their simplicity, robustness to hyperparameters, and strong overall performance on offline RL tasks. However, simply conditioning a probabilistic model on a desired return and taking the predicted action can fail dramatically in stochastic environments since trajectories that result in a return may have only achieved that return due to luck. In this work, we describe the limitations of RvS approaches in stochastic environments and propose a solution. Rather than simply conditioning on the return of a single trajectory as is standard practice, our proposed method, \algoshort{}, learns to cluster trajectories and conditions on average
cluster returns, which are independent from environment stochasticity. Doing so allows \algoshort{} to achieve strong alignment between target return and expected performance in real environments. We demonstrate this in several challenging stochastic offline-RL tasks including the challenging puzzle game 2048, and Connect Four playing against a stochastic opponent. In all tested domains, \algoshort{} achieves \textit{significantly} better alignment between the target return and achieved return than simply conditioning on returns. \algoshort{} also achieves higher maximum performance than even value-based baselines.
\end{abstract}

\input{sections/1_introduction}
\input{sections/2_theory}
\input{sections/3_method}
\input{sections/4_experiments}
\input{sections/5_related}
\input{sections/6_discussion}

\section*{Acknowledgements}

The authors thank Forest Yang, Harris Chan, Dami Choi, Joyce Yang, Sheng Jia, Nikita Dhawan, Jonah Philion, Michael Zhang, Qizhen Zhang, Lev McKinney, and Yongchao Zhou for insightful discussions and helpful feedback on drafts of the paper. We gratefully acknowledge funding from the Natural Sciences and Engineering Research Council of Canada (NSERC) [JB: 2020-06904], the Canada CIFAR AI Chairs Program, Microsoft Research, the Google Research Scholar Program and the Amazon Research Award. Resources used in preparing this research were provided, in part, by the Province of Ontario, the Government of Canada through CIFAR, and companies sponsoring the Vector Institute for Artificial Intelligence (\url{www.vectorinstitute.ai/partners}).

\bibliographystyle{unsrtnat}  % do not change this line!
% \balance  % do not change this line -- unless you manually balance your last page
\bibliography{esper}

\section*{Checklist}

%%% BEGIN INSTRUCTIONS %%%
% The checklist follows the references.  Please
% read the checklist guidelines carefully for information on how to answer these
% questions.  For each question, change the default \answerTODO{} to \answerYes{},
% \answerNo{}, or \answerNA{}.  You are strongly encouraged to include a {\bf
% justification to your answer}, either by referencing the appropriate section of
% your paper or providing a brief inline description.  For example:
% \begin{itemize}
%   \item Did you include the license to the code and datasets? \answerYes{See Section~\ref{gen_inst}.}
%   \item Did you include the license to the code and datasets? \answerNo{The code and the data are proprietary.}
%   \item Did you include the license to the code and datasets? \answerNA{}
% \end{itemize}
% Please do not modify the questions and only use the provided macros for your
% answers.  Note that the Checklist section does not count towards the page
% limit.  In your paper, please delete this instructions block and only keep the
% Checklist section heading above along with the questions/answers below.
%%% END INSTRUCTIONS %%%

\begin{enumerate}

\item For all authors...
\begin{enumerate}
  \item Do the main claims made in the abstract and introduction accurately reflect the paper's contributions and scope?
    \answerYes{}
  \item Did you describe the limitations of your work?
    \answerYes{See \autoref{par:limitations}.}
  \item Did you discuss any potential negative societal impacts of your work?
    \answerYes{See \autoref{par:impact}.}
  \item Have you read the ethics review guidelines and ensured that your paper conforms to them?
    \answerYes{}
\end{enumerate}

\item If you are including theoretical results...
\begin{enumerate}
  \item Did you state the full set of assumptions of all theoretical results?
    \answerYes{}
        \item Did you include complete proofs of all theoretical results?
    \answerYes{See the appendix in the supplemental materials.}
\end{enumerate}

\item If you ran experiments...
\begin{enumerate}
  \item Did you include the code, data, and instructions needed to reproduce the main experimental results (either in the supplemental material or as a URL)?
    \answerYes{Limited code is available in the supplemental material.}
  \item Did you specify all the training details (e.g., data splits, hyperparameters, how they were chosen)?
    \answerYes{See appendix.}
        \item Did you report error bars (e.g., with respect to the random seed after running experiments multiple times)?
    \answerYes{Experiments are run on multiple seeds with standard deviation error bars.}
        \item Did you include the total amount of compute and the type of resources used (e.g., type of GPUs, internal cluster, or cloud provider)?
    \answerYes{See appendix.}
\end{enumerate}

\item If you are using existing assets (e.g., code, data, models) or curating/releasing new assets...
\begin{enumerate}
  \item If your work uses existing assets, did you cite the creators?
    \answerNA{}
  \item Did you mention the license of the assets?
    \answerNA{}
  \item Did you include any new assets either in the supplemental material or as a URL?
    \answerNA{}
  \item Did you discuss whether and how consent was obtained from people whose data you're using/curating?
    \answerNA{}
  \item Did you discuss whether the data you are using/curating contains personally identifiable information or offensive content?
    \answerNA{}
\end{enumerate}

\item If you used crowdsourcing or conducted research with human subjects...
\begin{enumerate}
  \item Did you include the full text of instructions given to participants and screenshots, if applicable?
    \answerNA{}
  \item Did you describe any potential participant risks, with links to Institutional Review Board (IRB) approvals, if applicable?
    \answerNA{}
  \item Did you include the estimated hourly wage paid to participants and the total amount spent on participant compensation?
    \answerNA{}
\end{enumerate}

\end{enumerate}

%%%%%%%%%%%%%%%%%%%%%%%%%%%%%%%%%%%%%%%%%%%%%%%%%%%%%%%%%%%

%%%%%%%%%%%%%%%%%%%%%%%%%%%%%%%%%%%%%%%%%%%%%%%%%%%%%%%%%%%%

\newpage

\appendix
\input{sections/7_appendix}

\end{document}

%% file: _macros.tex
\usepackage{comment,url,algorithm,algorithmic,graphicx,subcaption,relsize}
\usepackage{amssymb,amsfonts,amsmath,amsthm,amscd,dsfont,mathrsfs,mathtools,nicefrac}
\usepackage{float,psfrag,epsfig,color,xcolor,url,hyperref}
\usepackage{listings}
\usepackage{color}

\definecolor{dkgreen}{rgb}{0,0.6,0}
\definecolor{gray}{rgb}{0.5,0.5,0.5}
\definecolor{mauve}{rgb}{0.58,0,0.82}

\def\balignst#1\ealignst{\begin{talign*}#1\end{talign*}}
\def\balignt#1\ealignt{\begin{talign}#1\end{talign}}
\let\originalleft\left
\let\originalright\right
\renewcommand{\left}{\mathopen{}\mathclose\bgroup\originalleft}
\renewcommand{\right}{\aftergroup\egroup\originalright}

\def\tinycitep*#1{{\tiny\citep*{#1}}}
\def\tinycitealt*#1{{\tiny\citealt*{#1}}}
\def\tinycite*#1{{\tiny\cite*{#1}}}
\def\smallcitep*#1{{\scriptsize\citep*{#1}}}
\def\smallcitealt*#1{{\scriptsize\citealt*{#1}}}
\def\smallcite*#1{{\scriptsize\cite*{#1}}}

\def\mbb#1{\mathbb{#1}}

\def\<{\left\langle} % Angle brackets
\def\>{\right\rangle}

\def\E{\mbb{E}} % Expectation symbol

\ifdefined\nonewproofenvironments\else
\ifdefined\ispres\else
\renewenvironment{proof}{\noindent\textbf{Proof.}\hspace*{.3em}}{\qed\\}
\newenvironment{proof-sketch}{\noindent\textbf{Proof Sketch}
  \hspace*{1em}}{\qed\bigskip\\}
\newenvironment{proof-idea}{\noindent\textbf{Proof Idea}
  \hspace*{1em}}{\qed\bigskip\\}
\newenvironment{proof-of-lemma}[1][{}]{\noindent\textbf{Proof of Lemma {#1}}
  \hspace*{1em}}{\qed\\}
\newenvironment{proof-of-theorem}[1][{}]{\noindent\textbf{Proof of Theorem {#1}}
  \hspace*{1em}}{\qed\\}
\newenvironment{proof-attempt}{\noindent\textbf{Proof Attempt}
  \hspace*{1em}}{\qed\bigskip\\}

\fi
\fi
\makeatletter
\@addtoreset{equation}{section}
\makeatother

\hypersetup{
  colorlinks,
  linkcolor={red!50!black},
  citecolor={blue!50!black},
  urlcolor={blue!80!black}
}

\newcommand{\handout}[5]{
  \noindent
  \begin{center}
    \framebox{
      \vbox{
        \hbox to 5.78in { {\bf \title } \hfill #2 }
        \vspace{4mm}
        \hbox to 5.78in { {\Large \hfill #5  \hfill} }
        \vspace{2mm}
        \hbox to 5.78in { {\em #3 \hfill #4} }
      }
    }
  \end{center}
  \vspace*{4mm}
}

%% file: macros-admin.tex
\definecolor{light-grey}{RGB}{120, 120, 120}
\definecolor{BrickRed}{RGB}{140, 0, 0}
\definecolor{LighterBlue}{RGB}{47, 140, 255}
\newif\ifcomments
\commentstrue
\ifcomments
\newcommand{\cut}[1]{}
\newcommand{\maybecut}[1]{\textcolor{orange}{#1}}
\newcommand{\debating}[1]{\textcolor{yellow}{#1}}

\newcommand{\donepromised}[1]{{}}
\newcommand{\toresolve}[1]{\textcolor{red}{#1}}
\newcommand{\remove}[1]{\textcolor{green}{\sout{#1}}}

\newcommand{\removehide}[1]{}
\newcommand{\alt}[1]{\textcolor{brown}{#1}}

\newcommand{\commentsm}[1]{\textcolor{BrickRed}{({\bf SM:} #1)}}

\newcommand{\commental}[1]{\textcolor{magenta}{({\bf AL:} #1)}}

\newcommand{\old}[1]{\textcolor{red}{\sout{#1}}}
\newcommand{\new}[1]{\textcolor{BrickRed}{\uline{#1}}}
\else
    \newcommand{\commentsm}[1]{}
    \newcommand{\commentrti}[1]{}
    \newcommand{\commentpv}[1]{}
    \newcommand{\commental}[1]{}
    \newcommand{\maybecut}[1]{}
    \newcommand{\debating}[1]{}
    
    \newcommand{\donepromised}[1]{}
    \newcommand{\toresolve}[1]{}
    \newcommand{\remove}[1]{}
    \newcommand{\removehide}[1]{}
    \newcommand{\alt}[1]{}
    \newcommand{\old}[1]{}
    \newcommand{\new}[1]{}

\fi

%% file: sections/1_introduction.tex
\section{Introduction}

Offline reinforcement learning (RL) \citep{levine_offline_2020, chen_decision_2021, kumar_conservative_2020, agarwal_optimistic_2020} is a promising approach to train agents without requiring online experience in an environment, which is desirable when online experience is expensive or when offline experience is abundant. A recent trend in offline RL has been to use simple approaches that do RL via Supervised Learning (RvS) (e.g., \citep{chen_decision_2021, emmons_rvs_2021, schmidhuber_reinforcement_2019, ghosh_learning_2021, kumar_reward-conditioned_2019}) rather than typical value-based approaches. RvS algorithms such as Decision Transformer \citep{chen_decision_2021} train a model to predict an action based on the current state and an outcome such as a desired future return. These agents ask the question \textit{``if I assume the desired outcome will happen, in my experience what action do I typically take next.''} These methods are popular due to their simplicity, strong performance on offline benchmark tasks \citep{chen_decision_2021, emmons_rvs_2021}, and similarity to large generative models (e.g., \citep{radford_language_2019, brown_language_2020, hoffmann_training_2022, alayrac_flamingo_2022}) that continue to show stronger performance on more tasks when training larger models on more data.

However, as we will show, \textbf{methods that condition on outcomes such as return can make incorrect decisions in stochastic environments regardless of scale or the amount of data they are trained on.} This is because implicitly these methods assume that actions that end up achieving a particular goal are optimal for achieving that goal. This assumption is not true in stochastic environments, where it is possible that the actions taken in the trajectory were actually sub-optimal and that the outcome was only achieved due to lucky environment transitions. For example, consider the gambling environment in \autoref{fig:toy}. Though there may be many episodes in which an agent gets a positive return from gambling ($a_0$ or $a_1$), gambling is sub-optimal since it results in a negative return in expectation while $a_2$ always results in a positive return. Since RvS takes all of these trajectories as expert examples of how to achieve the goal, RvS will act sub-optimally.

This limitation of RvS approaches in stochastic environments is well-hidden by the majority of benchmark tasks for offline RL (e.g., \citep{fu_d4rl_2020, gulcehre_rl_2020}), which tend to be deterministic or near-deterministic.  Locomotion tasks in MuJoCo \citep{todorov_mujoco_2012} and Atari games in the Arcade Learning Environment \citep{bellemare_arcade_2013} are two such examples.
While deterministic tasks can be solved by replaying promising action sequences, stochastic tasks are significantly less trivial to solve and often more realistic, requiring reactive policies to be learned \citep{machado_revisiting_2018}. Many real-world tasks are stochastic, either inherently or due to partial observability, such as having a conversation, driving a car, or navigating an unknown environment. In their current form, approaches such as Decision Transformer \citep{chen_decision_2021} are likely to behave in unexpected ways in such scenarios.

One way to view why RvS doesn't work in the gambling environment is that when conditioning on trajectories that achieve a positive reward, the model doesn't get to see any of the trajectories where the same sequence of actions leads to a negative reward. Due to these unrealistic dynamics, there is no policy that would generate this set of trajectories in the real environment, so it doesn't make sense to treat them as expert trajectories. Our insight is that there are certain functions of the trajectory other than return that, when conditioned on, will better preserve the dynamics of the environment. Our approach, called \algoshort{}, realizes this by \textit{conditioning on outcomes that are fully determined by the actions of the agent and independent of the uncontrollable stochasticity of the environment}. While trajectory return is not such an outcome, we show that the \textit{expected} return of behavior shown in a trajectory is, and how to learn such a value. The contributions of this work are as follows:

\begin{itemize}
    \item We show that RvS-style agents can only \textbf{\textit{reliably} achieve outcomes that are independent from environment stochasticity} and only depend on information that is under the agent's control.
    \item We propose a method to \textbf{learn environment-stochasticity-independent representations of trajectories using adversarial learning}, on top of which we label each trajectory with the average return for trajectories with this representation.
    \item \textbf{We introduce several stochastic offline RL benchmark tasks} and show that while RvS agents consistently underperform on the conditioned returns, our approach (\algoshort{}, short for \textbf{e}nvironment-\textbf{s}tochasticity-inde\textbf{pe}ndent \textbf{r}epresentations) \textbf{achieves significantly stronger alignment between target return and actual expected performance.} \algoshort{} gets \textbf{state-of-the-art performance} on these tasks, solving all tasks with near-maximum performance and surpassing the performance of even strong value-based methods such as CQL \citep{kumar_conservative_2020}.
\end{itemize}

\definecolor{verylightgrey}{rgb}{0.95,0.95,0.95}

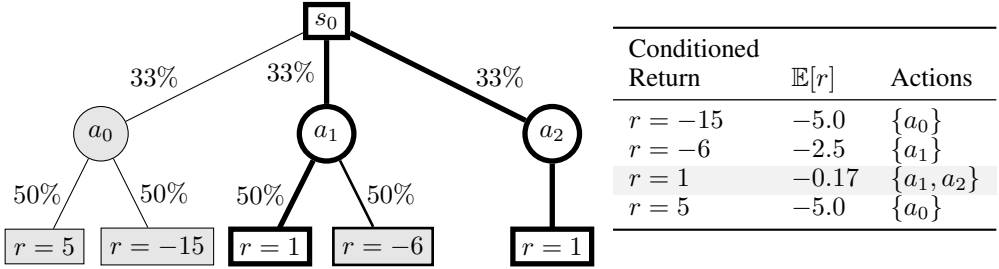
\begin{figure}
\centering
\begin{subfigure}{.47\textwidth}
  \centering
  \input{figures/toy}
  \label{fig:toy_task}
\end{subfigure}%
\hspace{2em}
\begin{subfigure}{.47\textwidth}
  \centering
  \begin{center}
\begin{tabular}{lll} 
 \toprule
 Conditioned\\Return & $\E[r]$ & Actions \\
 \midrule
 $r=-15$ & $-5.0$ & $\{a_0\}$\\
 $r=-6$ & $-2.5$ & $\{a_1\}$\\
 \rowcolor{verylightgrey}$r=1$ & $-0.17$ & $\{a_1, a_2\}$\\
 $r=5$ & $-5.0$ & $\{a_0\}$\\
 \bottomrule
\end{tabular}
\end{center}
\vspace{1em}
\end{subfigure}
\caption{\textbf{Left:} A simple gambling environment with three actions where return-conditioned algorithms such as Decision Transformer \citep{chen_decision_2021} will fail, \textit{even with infinite data}. The optimal action, $a_2$, will always grant the agent $1$ reward while gambling ($a_0$ and $a_1$) give a stochastic amount of reward. The numbers ($33\%$) above each action represent the data collection policy. \textbf{Right:} The second column represents the performance of a policy behaviorally cloned from trajectories achieving the reward in the first column. For example, conditioned on the agent receiving $r=1$, a third of trajectories take $a_1$ and two-thirds take $a_2$. Averaging the returns $-2.5 \times 1/3 + 1 \times 2/3 = -0.17$. \textbf{No matter how much data the model is trained on, when conditioned on receiving a reward of $1$, it will always gamble and take $a_1$ some of the time, rather than just taking $a_2$, which guarantees the reward.}} 
\label{fig:toy}
\end{figure}

%% file: figures/toy.tex
\begin{tikzpicture}

\definecolor{highlight}{rgb}{1,1,1}
\definecolor{background}{rgb}{0.9,0.9,0.9}

\node[rectangle,draw,line width=2pt,fill=highlight] {$s_0$}[sibling distance = 3cm]
    child {node[circle,draw,fill=background] {$a_0$}[sibling distance = 1.5cm]
        child {node [rectangle,draw,solid,fill=background] {$r=5$}
                edge from parent [solid] node [left] {$50\%$}}
        child {node [rectangle,draw,solid,fill=background] {$r=-15$} 
                edge from parent [solid] node [right] {$50\%$}}
            edge from parent [solid] node [left,xshift=-10pt] {$33\%$}}
    child {node[circle,draw,line width=2pt,fill=highlight] {$a_1$}[sibling distance = 1.5cm]
        child {node [rectangle,draw,fill=highlight] {$r=1$}
                edge from parent [solid,line width=2pt] node [left] {$50\%$}}
        child {node [rectangle,draw,fill=background,line width=1pt] {$r=-6$}
                edge from parent [solid,line width=1pt] node [right] {$50\%$}}
            edge from parent [solid,line width=2pt] node [left] {$33\%$}}
    child {node[circle,draw,line width=2pt,fill=highlight] {$a_2$}[sibling distance = 1.5cm]
        child {node [rectangle,draw,fill=highlight] {$r=1$}}
            edge from parent [solid,line width=2pt] node [right,xshift=10pt] {$33\%$}};

\end{tikzpicture}

%% file: sections/2_theory.tex
\section{Approach}

\subsection{Problem Setup}

We model the environment as an MDP, defined as the tuple $(S, A, T, R, \gamma)$. $S$ is a set of states; $A$ is a set of actions; the transition probabilities $T: S \times A \times S \to [0, 1]$ defines the probability of the environment transitioning from state $s$ to $s'$ given that the agent acts with action $a$; the reward function $R: S \times A \to \mathbb{R}$ maps a state-action transition to a real number; and $0 \leq \gamma \leq 1$ is the discount factor, which controls how much an agent should prefer rewards sooner rather than later. The performance a policy, defined by $\pi(a|s)$, is typically measured using the cumulative discounted return $\sum_t \gamma^t r_t$ that the policy achieves in an episode.

Central to our work is algorithms that do reinforcement learning via supervised learning (RvS) such as Decision Transformer \citep{chen_decision_2021}. These approaches train a model using supervised learning on a dataset of trajectories to predict $p_{\mathcal{D}}(a|s, R)$ - the probability of next action conditioned on the current state and on the agent getting a cumulative discounted return $R = \sum_t \gamma^t r_t$. At evaluation time, the model is conditioned on a desired target return and the agent takes the actions predicted by the model.

To understand the problem with RvS in stochastic environments, we build off of a general framework where a goal is presented to a policy, the policy acts in the environment, and the agent is scored on how well the goal was achieved as described in \citet{furuta_generalized_2022}. Mathematically, the agent's objective can be expressed as minimizing some distance metric $D(\cdot, \cdot)$ between a target goal $z$ and the \textit{information statistics} $I(\tau)$ of the trajectory generated by a conditional policy $\pi(a|s, z)$:

\vspace{-1em}
\begin{equation}
    \min_{\pi} \E_{z \sim p(z), \tau \sim p_{\pi_z}(\tau)}\left[D(I(\tau), z)\right]
    \label{eq:im}
\end{equation}
where $p(z)$ is the goal distribution, $p_{\pi_z}(\tau)$ is the trajectory distribution $(s_1, a_1, s_2, a_2, \ldots)$ obtained by running rollouts in an MDP following policy $\pi(a|s, z)$. In this section, we refer to the goal presented to an agent as $z$ and the function that calculates which goal was achieved in a particular trajectory $\tau$ as its statistics $I(\tau)$. As discussed in \citet{furuta_generalized_2022}, this framework includes a number of popular algorithms, including RvS algorithms that condition on target returns (e.g., \citep{chen_decision_2021, emmons_rvs_2021, kumar_reward-conditioned_2019, schmidhuber_reinforcement_2019}) or goals (e.g., \citep{ghosh_learning_2021, paster_planning_2021}). For example, in Decision Transformer \citep{chen_decision_2021} cumulative discounted reward is used as the trajectory statistics $I(\tau) = \sum_t \gamma^t r_t$ and a reasonable choice for $D(\cdot, \cdot)$ could be the squared error between the target and achieved return.

\citet{furuta_generalized_2022} justify using the supervised learning approach (RvS) to minimize \autoref{eq:im} by claiming that trajectories with statistics $I(\tau)$ act as \textit{optimal} examples of how to act to achieve $D(I(\tau), z) = 0$ for an agent whose goal is to achieve $z = I(\tau)$. However, in stochastic environments, the actions taken in these trajectories may not be optimal since \textbf{a trajectory may achieve a goal due to uncontrollable randomness in the environment rather than from the agent's own actions.} In fact, as discussed in \citet{paster_planning_2021}, not only can the actions taken in these ``optimal trajectories'' actually be sub-optimal, but they can be arbitrarily bad depending on the policy used to collect the data, the environment dynamics, and the choice of distance metric.

In this section, we propose to sidestep this issue by limiting the types of statistics $I(\tau)$ we are interested in to those that are not affected by uncontrollable randomness in the environment. We show first that RvS policies trained to reach such goals, under the assumption of infinite data and model capacity, learn optimal policies for achieving such goals and that goals that are independent from environment stochasticity are the only goals that RvS can reliably achieve. Importantly, we argue that limiting the form of trajectory statistics $I(\tau)$ in this way is not prohibitive and that desirable quantities can be transformed to fit this property in an intuitive way (e.g., return turns into expected return when optimized to be independent of environment stochasticity). Finally, we propose an approach for automatically finding such trajectory statistics using adversarial learning.

\subsection{Stochasticity Independent Representations}

\begin{definition}[Consistently Achievable]
A goal $z$ is \textit{consistently achievable} from state $s_0$ under policy $\pi(a|s, z)$ if $\E_{\tau \sim p_{\pi_z}(\tau|s_0)}\left[D(I(\tau), z)\right] = 0$.
\end{definition}

A natural question is under which circumstances will the RvS approach actually minimize \autoref{eq:im} to zero? Clearly the approach works empirically on deterministic environments \citep{chen_decision_2021, emmons_rvs_2021} since simply replaying an action sequence that achieved a goal once will achieve it again. While not all statistics of a trajectory will lead to consistently achievable goals under policies trained with RvS, \textbf{the supervised learning approach will minimize \autoref{eq:im} if $I(\tau)$ is independent from environment stochasticity}.

\begin{theorem}
Let $\pi_{\mathcal{D}}$ be a data collecting policy used to gather data used to train an RvS policy, and assume this policy is trained such that $\pi(a_t|s_0, a_0, \ldots, s_t, z) = p_{\pi_{\mathcal{D}}}(a_t|s_0, a_0, \ldots, s_t, I(\tau) = z)$. Then, for any goal $z$ such that $p_{\pi_{\mathcal{D}}}(I(\tau)=z|s_0) > 0$, goal $z$ is \textit{consistently achievable} iff $p_{\pi_{\mathcal{D}}}(s_t|s_0, a_0, \ldots, a_{t-1}) = p_{\pi_{\mathcal{D}}}(s_t|s_0, a_0, \ldots, a_{t-1}, I(\tau) = z)$.
\label{thm:main}
\end{theorem}

\begin{proof}[Sketch]
The forward direction of the theorem is proven by rewriting $\E_{\tau \sim p_{\pi_z}(\tau|s_0)}\left[D(I(\tau), z)\right]$ to be taken over $p_{\pi_z}(\tau|s_0, I(\tau) = z)$ using the independence assumption combined with the form of the policy. For the other direction and the full proof, refer to Appendix \ref{sec:proof}.
\end{proof}

{\bf{Remarks. }} The above theorem has two major implications. First is that in stochastic environments, many of the most common statistics such as final states or cumulative discounted rewards are often correlated with environment stochasticity and therefore are unlikely to work well with the behavioral cloning approach. For example, in the gambling environment discussed in \autoref{fig:toy}, the achieved reward is correlated with the uncontrollable randomness present when gambling, and therefore when filtering an offline dataset for trajectories that achieve a high reward, the resulting trajectories will have unrealistic environment dynamics because we ignore the unsuccessful attempts (i.e., the agent always wins the money when gambling) and an agent behavioral cloned on these trajectories will not actually end up achieving a high reward consistently.

\autoref{thm:main} also gives a solution to ensure asymptotically that RvS policies will consistently achieve goals: use trajectory statistics that are independent from the stochasticity in environment transitions. This requires the statistic function $I(\tau)$ to return the same value for different trajectories where the only difference is in the environment transitions, making $I(\tau)$ essentially act as a cluster assignment where each trajectory within a cluster has the same value of $I(\tau)$ and has environment transitions sampled from the true distribution $p(s_{t+1}|s_t, a_t)$. These clusters can be thought of as datasets each generated by distinct policies in the same environment. For example, in the gambling environment each cluster could include trajectories that all took the same action.

Ultimately, we want to use a target \textit{average return} to control an RvS agent. Luckily, given any $I(\tau)$ that satisfies \autoref{thm:main}, we can transform it by using the average returns of trajectories within each cluster as the statistic instead. In the next section, we describe a practical algorithm for learning an $I(\tau)$ that satisfies \autoref{thm:main}, finding the average return of trajectories with the same value of $I(\tau)$ (i.e. in the same ``cluster''), and training an RvS policy conditioned on these estimated average returns.

%% file: sections/3_method.tex
\subsection{Learning Stochasticity-Independent Representations}

\begin{figure}[t]
    \centering
    \includegraphics[width=1\textwidth]{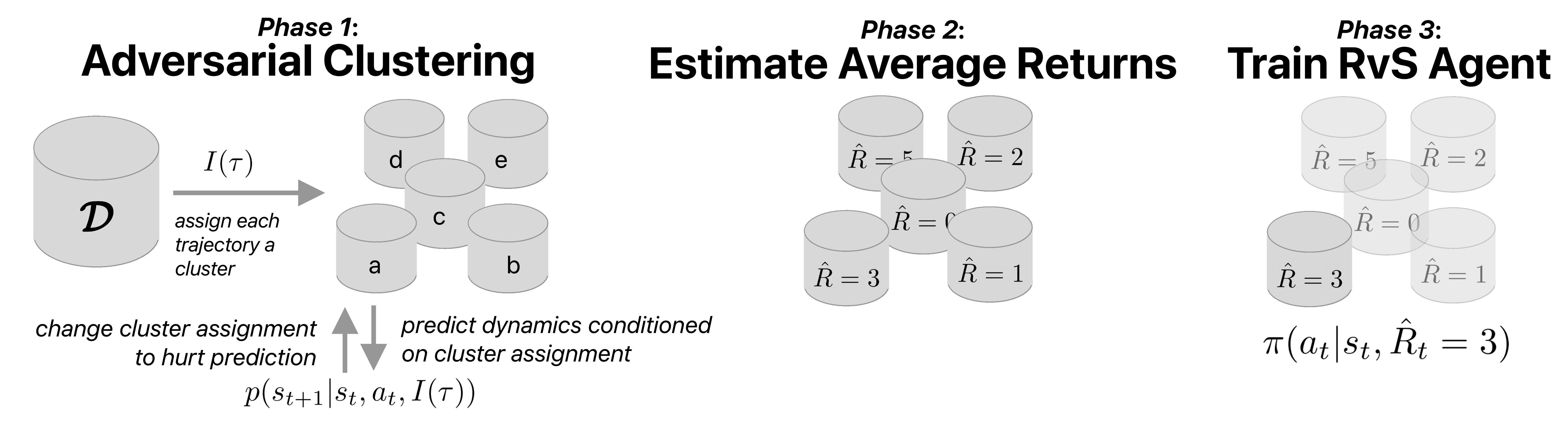}
    \caption{\algoshort{} learns a policy that conditions on a desired \textit{expected} return. In phase 1 of the algorithm, a function is learned using adversarial learning that assigns each trajectory in the dataset to a cluster such that the cluster assignments do not contain information about the stochastic outcomes of the environment that can help a dynamics model cheat to predict next states. In phase 2, the average return in each cluster is calculated. In phase 3, an RvS agent is trained to predict the next action given the current state and the estimated average return.\vspace{-1em}}
    \label{fig:algo_fig}
\end{figure}

In this section we propose a method for learning such trajectory representations, which we call \algoshort{}, short for \textbf{e}nvironment-\textbf{s}tochasticity-inde\textbf{pe}ndent \textbf{r}epresentations. Rather than use trajectory return as the statistic $I(\tau)$, which is often dependent on environment stochasticity and therefore may result in unexpected behavior, \algoshort{} uses a neural network to learn statistics that are independent from environment stochasticity. It does so in three phases: first, it uses an auto-encoder framework to learn a discrete representation $I(\tau)$ acting as a cluster assignment for each trajectory along with an adversarial loss that changes the representation in order to hurt the predictions of a dynamics model; second, a model learns to predict average trajectory return from the learned representation; and third, an RvS agent is trained to predict actions taken when conditioned on a state and estimated average return. 

For some trajectory $\tau = (s_1, a_1, r_1, s_2, a_2, r_2, \ldots)$, \algoshort{} trains the following parameterized models in addition to a vanilla RvS policy $\pi_{\xi}$:

\vspace{-1em}
\begin{align*}
    & \text{Clustering model:} & I(\tau) &\sim p_{\theta}(I(\tau)|\tau)\\
    & \text{Action predictor:} & a_t &\sim p_{\theta}(a_t|s_t, I(\tau))\\
    & \text{Return predictor:} & \hat{R} &= f_{\psi}(I(\tau))\\
    & \text{Transition predictor:} & s_{t+1} &\sim p_{\phi}(s_{t+1}|s_0, a_0, \ldots, s_t, a_t, I(\tau))
\end{align*}

{\bf{Adversarial clustering.}} Here, we will focus on learning representation in the form of clustering assignments. The goal is to learn a clustering assignment per trajectory $I(\tau)$ to minimize the discrepancy between the true environment transition $p(s_{t+1}|s_t, a_t)$ and the estimated transition $p_{\phi}(s_{t+1}|s_t, a_t, I(\tau))$ from the replay buffer when conditioned on $I(\tau)$. Such a representation will satisfy~\autoref{thm:main}. In the first phase of training, discrete cluster assignments produced by the clustering model are trained by taking alternating gradient steps with the following two losses to ensure that information present in $I(\tau)$ does not improve the performance of the dynamics model:
\begin{align}
    L(\theta) &= \E_{I(\tau) \sim p_{\theta}(I(\tau)|\tau)}\left[\describe{-\beta_{\text{act}}\log p_{\theta}(a_t|s_t, I(\tau))}{policy reconstruction} + \describe{\beta_{\text{adv}}\log p_{\phi}(s_{t+1}|s_0, a_0, \ldots, s_t, a_t, I(\tau))}{adversarial loss}\right]\label{eq:rep_loss}\\
    L(\phi) &= \E_{I(\tau) \sim p_{\theta}(I(\tau)|\tau)}\left[\describe{-\log p_{\phi}(s_{t+1}|s_0, a_0, \ldots, s_t, a_t, I(\tau))}{dynamics prediction}\right]\label{eq:dyn_loss}
\end{align}
In this step, the dynamics prediction loss is predicting the next state given the current state, action, and cluster assignment $I(\tau)$ while the adversarial loss tries to change $I(\tau)$ in order to hurt this prediction. Since a constant representation minimizes this loss, we use a policy reconstruction loss to encourage $I(\tau)$ to contain information about the policy used to generate the trajectory and encourage the formation of more than one cluster. $\beta_{\text{act}}$ and $\beta_{\text{adv}}$ are hyperparameters to balance the strength of this policy reconstruction loss with the adversarial loss.

{\bf{Estimate cluster average returns.}} After clustering, we learn to predict discounted returns $R = \sum_t \gamma^t r_t$ for each cluster.\footnote{Note that in practice, we condition the return predictor not only on the cluster assignment for the trajectory, but also on the first state and action. This improves performance while still satisfying \autoref{thm:main}.} The return predictor parameterized by $\psi$ is trained using the following loss:

\begin{equation}
    L(\psi) = \E_{I(\tau) \sim p_{\theta}(I(\tau)|\tau)}\left[\|R - f_{\psi}(I(\tau))\|_2^2\right].\label{eq:return_loss}
\end{equation}
The above losses are described such that the model produces a single estimated return per trajectory. However, in practice it can be extended to produce a value for each time-step in a trajectory by using suffixes of trajectories $\tau_{i, t} = (a_t, s_t, r_t, \ldots, a_T, s_T, r_T)$ to generate return predictions $\hat{R}_t$ for each time-step.

{\bf{Training policy on predicted returns.}} After learning expected future returns for each step in a trajectory, we use the dataset of $(s_t, a_t, \hat{R}_t)$ triples to learn a policy $\pi_{\xi}(a_t|s_t, \hat{R}_t)$ parameterized by $\xi$ by predicting the action from the current state and estimated expected return by minimizing the following loss just as in prior RvS works \citep{emmons_rvs_2021, chen_decision_2021}:

\begin{equation}
    L(\xi) = \E_{s_t, a_t, \hat{R}_t \sim \mathcal{D}}\left[-\log \pi_{\xi}(a_t|s_t, \hat{R}_t)\right]\label{eq:policy_loss}
\end{equation}

\subsection{Implementation}

The clustering model is implemented using an LSTM \citep{hochreiter_long_1997}, using truncated backpropagation through time for training on long trajectories (up to 1000 time-steps in our experiments). The action predictor, return predictor, and transition predictor are all implemented using MLPs. In practice, we found that the dynamics predictor works well even without conditioning on the entire trajectory history. We use a transformer policy as in Decision Transformer \citep{chen_decision_2021} (see \autoref{sec:baselines}), although we find that using an MLP is sufficient for strong performance in our benchmark tasks (see appendix \ref{sec:mlp}). Our clustering outputs a discrete value sampled from a categorical distribution and we use gumbel-softmax \citep{jang_categorical_2017, maddison_concrete_2017} in order to backpropagate gradients through it. We use normal distributions with unit variances for the transition predictor. Code for our implementation is available at \href{https://sites.google.com/view/esper-paper}{https://sites.google.com/view/esper-paper}. See appendix \ref{sec:code} for pseudocode for the adversarial clustering step of \algoshort{}. See appendix \ref{sec:training_details} for more information about the implementation, including specific hyperparameters for our environments as well as best practices.

%% file: sections/4_experiments.tex
\section{Experiments}

\begin{figure}[t]
    \centering
    \begin{center}
\begin{tabular}{llll} 
 \toprule
 Task & Return-Conditioned RvS (DT) & CQL & \algoshort{} (Ours) \\
 \midrule
 Gambling & -0.02 (0.24) & \textbf{1.0 (0.0)} & \textbf{1.0 (0.0)}\\ 
 Connect Four & 0.8 (0.07) & 0.61 (0.05) & \textbf{0.99 (0.03)} \\ 
 2048 & 0.57 (0.05) & 0.7 (0.09) & \textbf{0.81 (0.06)} \\
%  Half Cheetah (G) & \textbf{491.28 (32.4)} & \textbf{475.57 (21.85)} & 0 \\
 \bottomrule
\end{tabular}
\end{center}
    \caption{Comparing the maximum performance of RvS with return conditioning, RvS with \algoshort{} conditioning (ours), and CQL \citep{kumar_conservative_2020}, a strong value-based offline-RL baseline. The maximum possible expected return for a policy is $1.0$ on all tasks. Since all our tasks are stochastic, return conditioning fails to learn an optimal policy while \algoshort{} performs optimally or near optimally. CQL performs well on the gambling task, but cannot achieve the same level of performance as \algoshort{} on the more complicated tasks. Numbers in parenthesis are standard deviations over $3$ seeds.\vspace{-1em}}
    \label{fig:perf_table}
\end{figure}

We designed our experiments to answer the following questions: \textbf{(i)} How well does return-conditioned RvS work on stochastic problems where returns are correlated with environment stochasticity? \textbf{(ii)} Does \algoshort{} improve performance, both in terms of maximum achievable performance as well as in correlation between the target and achieved returns? \textbf{(iii)} Does the degree to which the learned representations in \algoshort{} are independent from environment stochasticity measurably affect the performance of the learned agent?

\subsection{Stochastic Benchmark Tasks}

Prior offline RL methods including Decision Transformer have been tested primarily on deterministic or near-deterministic environments such as continuous control tasks from the D4RL dataset \citep{fu_d4rl_2020} and Atari games \citep{bellemare_arcade_2013}. To get a better idea of how these methods and our approach work on realistic, stochastic environments, we created three new benchmark tasks. Since we are evaluating in an offline setting, each task consists of a stochastic environment as well as one or several data collection policies in order to gather the offline dataset.

{\bf Gambling.} An illustrative gambling environment described in \autoref{fig:toy} with only $3$ actions. The offline dataset for this task is 100k steps collected with a random policy.

{\bf Connect Four.} A game of Connect Four, a tile-based game where players compete to get four-in-a-row, against an opponent that does not place a tile in the rightmost column with a probability of $20\%$. The agent gets $1$ reward for winning, $0$ for a draw, and $-1$ for losing. The offline dataset for this task is 1M steps collected using a mixture of an $\epsilon$-greedy policy and a policy that always places tiles in the rightmost column.

{\bf 2048.} A simplified version of 2048 \citep{cirulli_2048_2014}, a sliding puzzle game where alike numbers can be combined to create tiles of larger value, where the agent gets a reward of $1$ by creating a 128 tile and gets no reward otherwise. The offline dataset for this task is 5M steps collected using a mixture of a random agent and an expert policy trained using PPO \citep{schulman_proximal_2017}.

Detailed information about each environment can be found in Appendix \ref{sec:benchmark_details}.

\begin{figure}
\centering
\begin{subfigure}{.5\textwidth}
  \centering
  \includegraphics[width=0.6\textwidth]{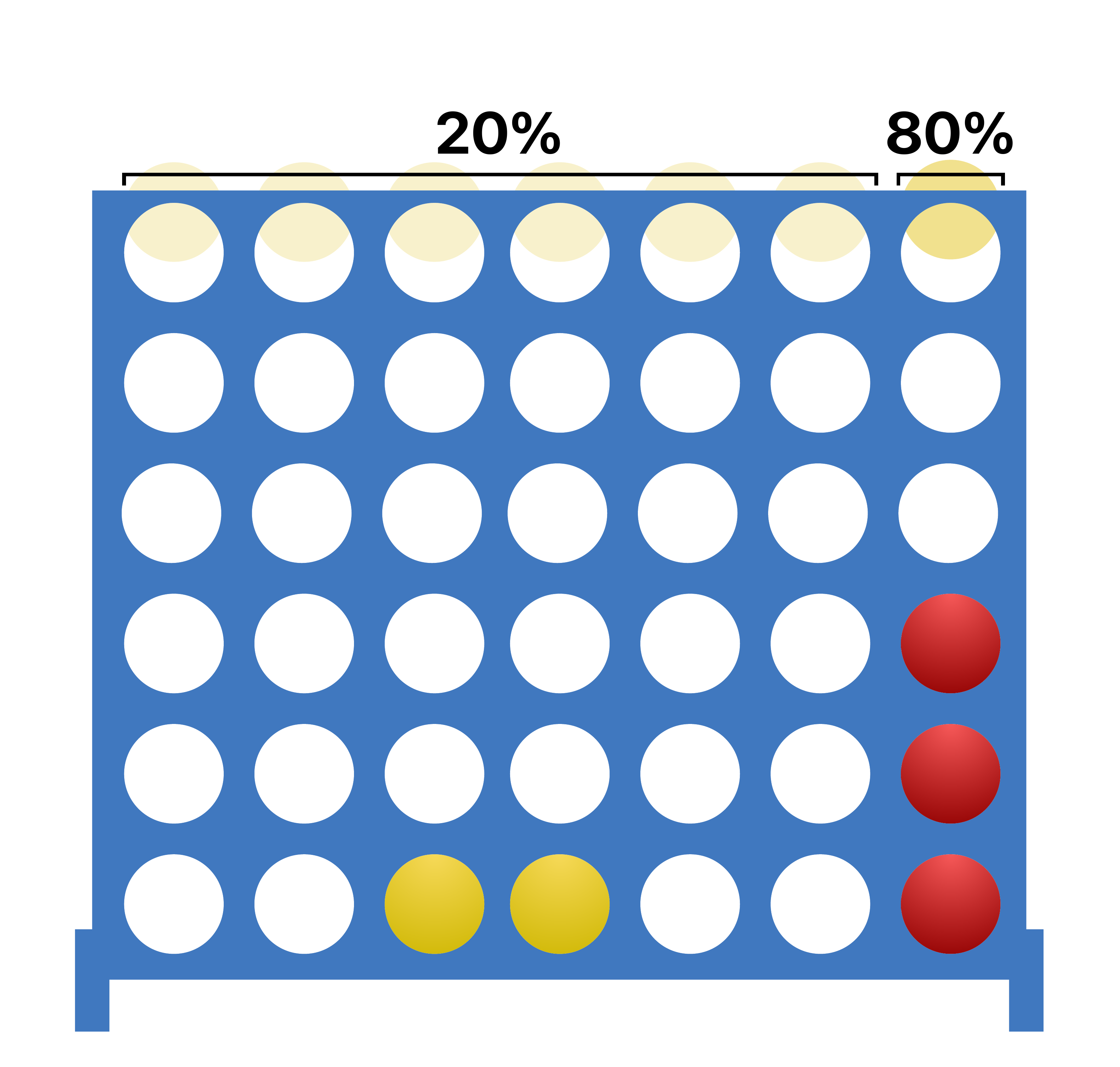}
  \vspace{1.3em}
  \label{fig:c4_pic}
\end{subfigure}%
% \hspace{7em}
\begin{subfigure}{.5\textwidth}
  \centering
  \includegraphics[width=1\textwidth]{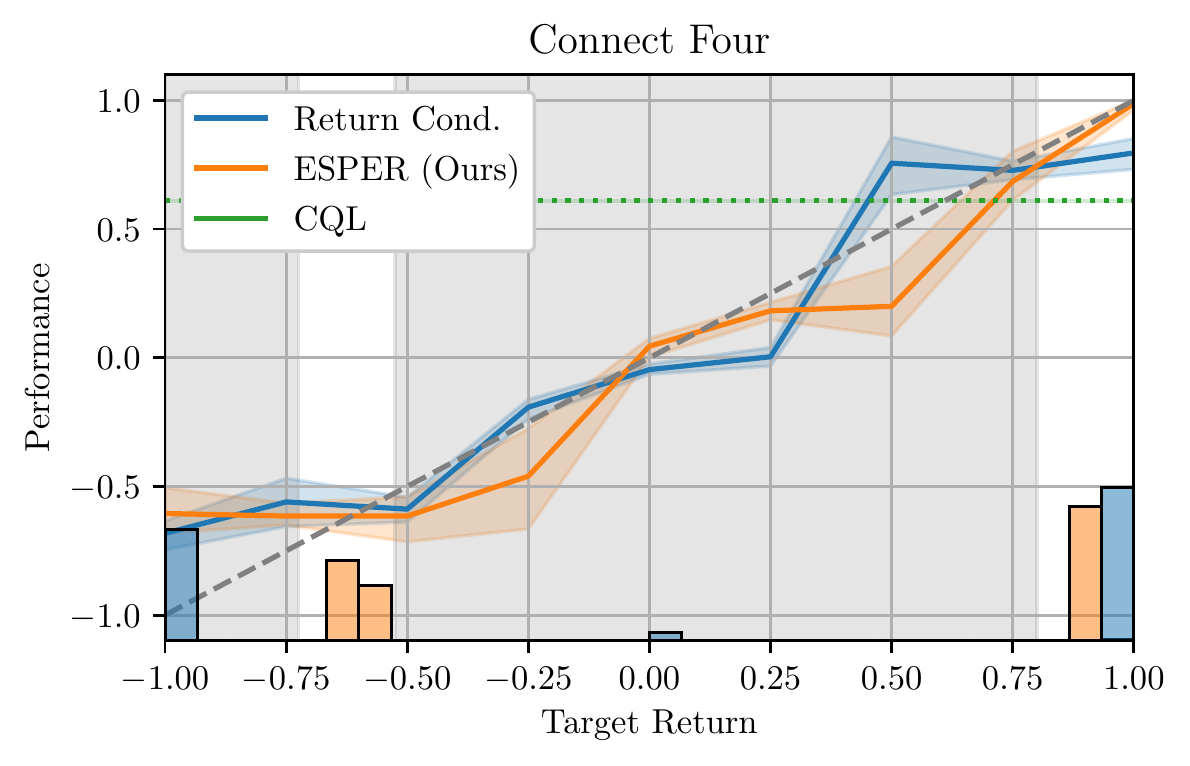}
\end{subfigure}
\caption{We consider Connect Four with a stochastic opponent that fails to place a piece in the rightmost column with a $20\%$ chance. In this environment, the agent gets a reward of $1$ for winning, $0$ for a draw, and $-1$ for losing. The histogram represents the distribution of returns each method is trained on and out-of-distribution regions for \algoshort{} are shaded. On in-distribution returns, \algoshort{} achieves performance (y-axis) closer to the target performance (x-axis) than return conditioning.\vspace{-1em}}
\label{fig:c4_plot}
\end{figure}

\subsection{Baselines}
\label{sec:baselines}

We compare our method primarily against Decision Transformer \citep{chen_decision_2021}, which trains a transformer \citep{vaswani_attention_2017} to predict the next action conditioned on a sequence of previous states, actions, and return-to-go targets. Our method uses the same model and training code, except rather than condition on target returns we condition on expected returns learned using \algoshort{}. While we conduct our experiments using transformers due to the popularity of Decision Transformer \citep{chen_decision_2021}, we note that our observations still hold when using MLP policies as in \citet{emmons_rvs_2021} while using significantly less compute and getting similar performance (see Appendix \ref{sec:mlp}). We also compare against a strong value-based approach for offline-RL called Conservative Q-Learning (CQL) \citep{kumar_conservative_2020}, which has been found to have strong performance on both discrete and continuous offline RL tasks. 

\subsection{Performance in Stochastic Environments}

\begin{figure}
\centering
\begin{subfigure}{.5\textwidth}
  \centering
  \includegraphics[width=1\textwidth]{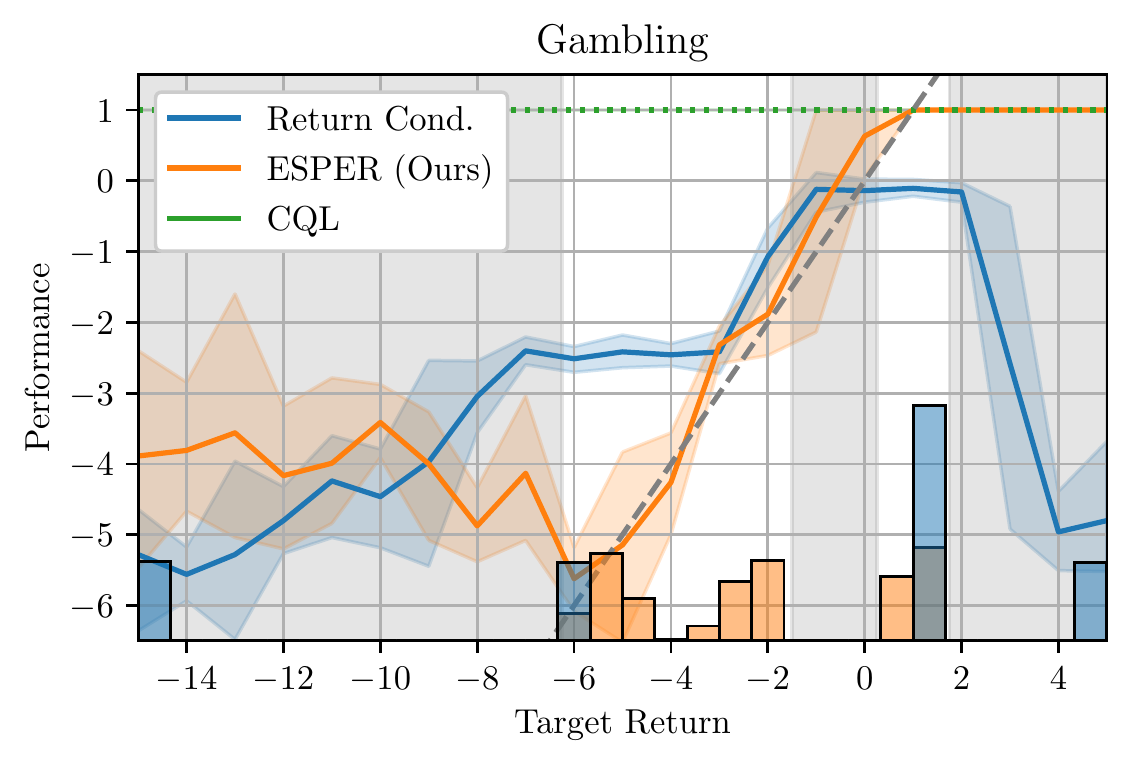}
  \vspace{-1.5em}
  \caption{}
  \label{fig:toy_plot}
\end{subfigure}%
\begin{subfigure}{.5\textwidth}
  \centering
  \includegraphics[width=1\textwidth]{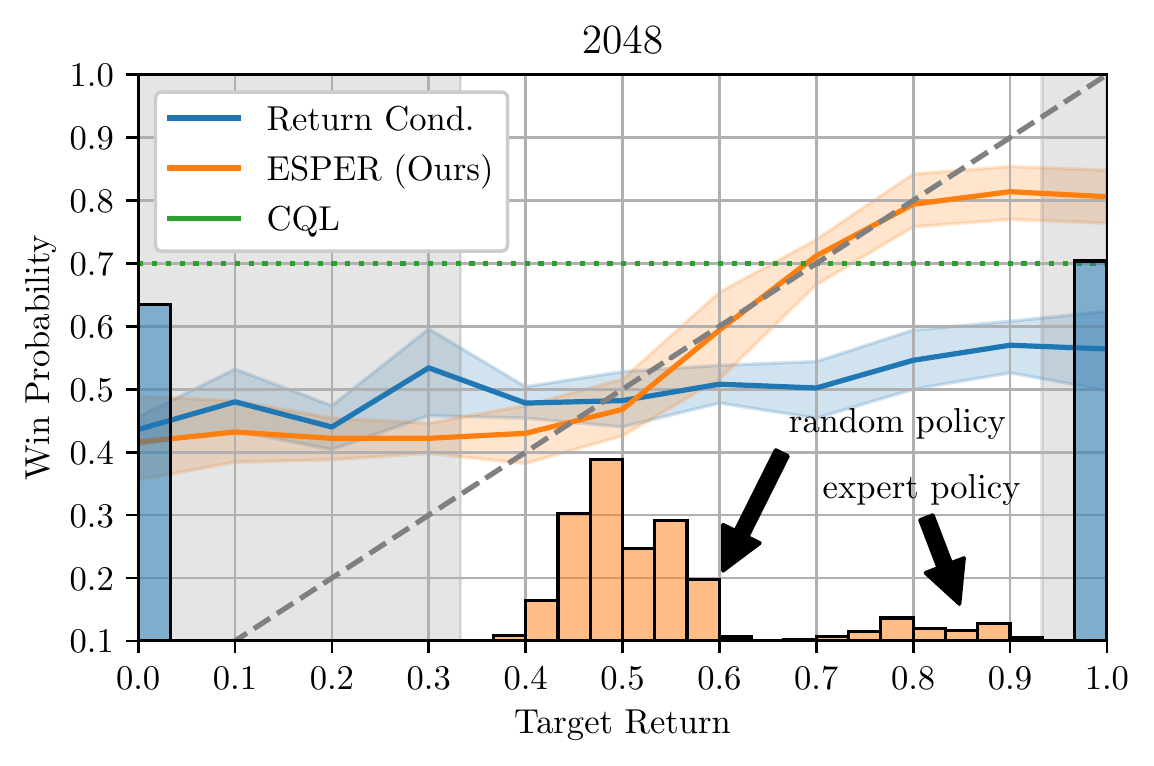}
  \vspace{-1.4em}
  \caption{}
  \label{fig:2048_plot}
\end{subfigure}
\caption{The histogram represents the distribution of returns each method is trained on. \textbf{Left:} In the illustrative gambling environment, \algoshort{} can achieve a range performances from $-5$ to $1$, while the performance (y-axis) of the return-conditioned agent is not aligned with the target return (x-axis). \textbf{Right:} In our modified 2048 task, the agent receives a reward of $1$ for creating a tile of value $128$ and $0$ otherwise. While a return-conditioned agent does not achieve a high level of performance, the \algoshort{} agent disentangles actions from environment stochasticity and is able achieve performance close to the target return. Interestingly, the distribution of expected returns learned by \algoshort{} match the performance-level of the data-collection policies.\vspace{-1em}}
\end{figure}

When comparing the performance of our approach, \algoshort{}, with return-conditioned RvS, we found that \textbf{\algoshort{} consistently achieves stronger alignment between target return and average performance} while also achieving a higher maximum level of performance when tuning the target return. As shown in \autoref{fig:toy_plot}, the return-conditioned baseline cannot achieve the maximum performance of $1$ while \algoshort{} learns behaviors corresponding to average performances ranging from $-5$ to $1$. In Connect Four (\autoref{fig:c4_plot}), despite seeing many examples of winning trajectories, an RvS agent cannot achieve more than $0.2$ average return (corresponding to a win-rate of $60\%$) while \algoshort{} can both win and lose with perfect accuracy depending on the target performance. In 2048 (\autoref{fig:2048_plot}), RvS cannot disentangle trajectories where the reward is high simply from luck from trajectories that took good actions and can only win the game about $60\%$ of the time. \algoshort{} learns many modes of behavior, and can be controlled to win anywhere from $30\%$ to $80\%$ of the time.

% \begin{figure}
%     \centering
%     \includegraphics[width=1\textwidth]{figures/data.pdf}
%     \caption{While a return-conditioned agent does not improve its performance with more data, \algoshort{} achieves performance (y-axis) closer to the target return (x-axis) when trained on more data. $100\%$ data usage is 5 million frames.\vspace{-1em}}
%     \label{fig:data_plot}
% \end{figure}

\begin{figure}
\centering
\begin{subfigure}{.5\textwidth}
  \centering
  \vspace{0.5em}
  \includegraphics[width=1\textwidth]{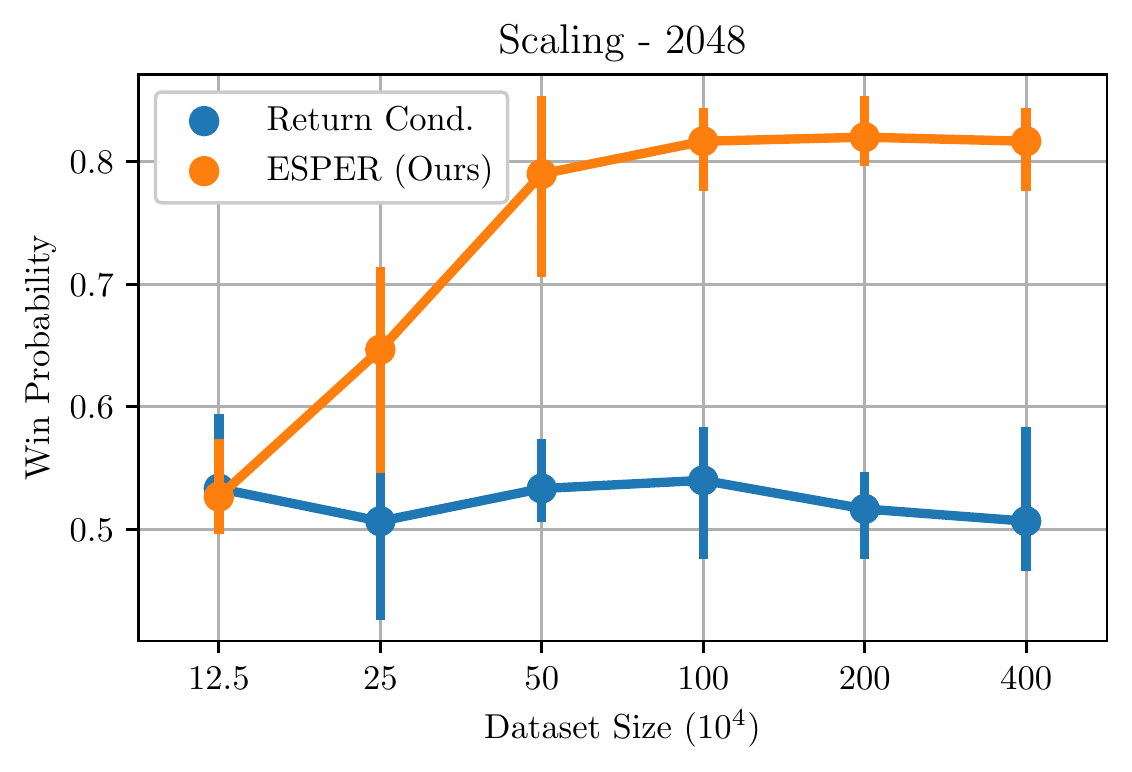}
  \vspace{-1.5em}
%   \caption{} add this back TODO for arxiv and camera ready!!
  \label{fig:data_plot}
\end{subfigure}%
\begin{subfigure}{.5\textwidth}
  \centering
  \includegraphics[width=1\textwidth]{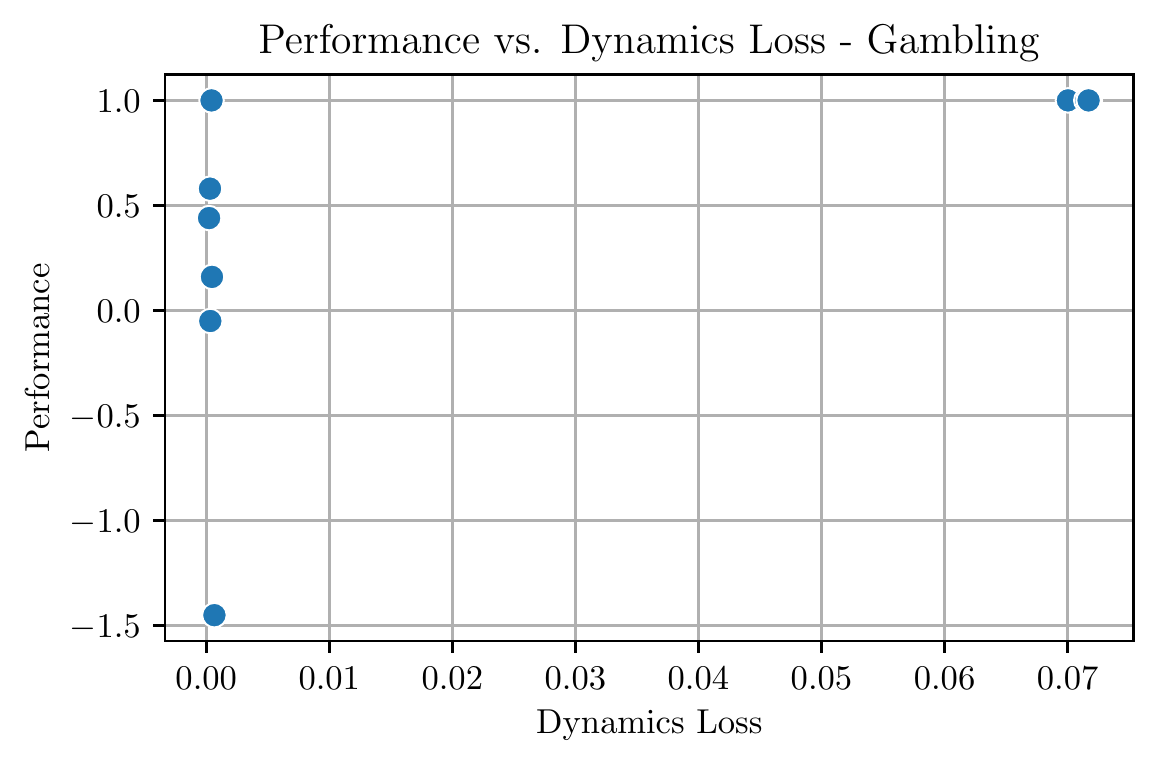}
  \vspace{-1.1em}
%   \caption{} 
  \label{fig:adv_plot}
\end{subfigure}
\caption{\textbf{Left:} While a return-conditioned agent does not improve its performance with more data, \algoshort{} achieves performance (y-axis) closer to the target return (x-axis) when trained on more data. $100\%$ data usage is 5 million frames. \textbf{Right:} We trained several agents with different hyperparameters on the gambling task. Notably, agents with trajectory statistics that enabled the dynamics model to ``cheat'' and get a low loss performed worse than those with representations independent from environment stochasticity, empirically confirming \autoref{thm:main}.\vspace{-1em}}
\label{fig:data_adv_plot}
\end{figure}

Additionally, \autoref{fig:data_adv_plot} shows that \textbf{this lack of performance for the return-conditioned agent is not due to a lack of data}. In the figure, going from $5\%$ data to $100\%$ (5 million frames) does not affect the performance of RvS while the performance of \algoshort{} gets closer to matching the target return with more data. Finally, we compared the maximum performance of RvS and \algoshort{} with the state-of-the-art value-based offline RL algorithm, Conservative Q-Learning (CQL) \citep{kumar_conservative_2020}. As shown in \autoref{fig:perf_table}, \algoshort{} and CQL achieve perfect performance on the gambling task while RvS cannot get positive reward. In Connect 4 and 2048, while CQL can get respectable performance (better than the average return of the offline dataset), \textbf{\algoshort{} achieves state-of-the-art performance}.
\vspace{-0.5em}
\subsection{Learned Representations and Behaviors}

Rather than simply condition on returns, \algoshort{} conditions on learned expected return values. In addition to plotting the performance against the target return, we also use a histogram to show the distribution of values (returns or learned expected returns) on which the agents are trained. As shown in the histograms (\autoref{fig:c4_plot}, \autoref{fig:toy_plot}, \autoref{fig:2048_plot}), \algoshort{} often learns values for many more modes of behaviors corresponding to many different expected returns. Unlike RvS, when an agent is trained on a particular expected return value, the actual average performance of the agent is often close to this value. 
This is useful for offline settings where tuning the target return in an online environment is not feasible, since \textbf{with \algoshort{} one can have confidence in which levels of performance the model will be able to achieve}. In contrast, with return-conditioning, the performance of the return-conditioned policy may not correlate with the returns on which it was trained.

Finally, we empirically measured the relationship between the independence of learned trajectory statistics and the performance of the agent. As shown in \autoref{fig:data_adv_plot}, we trained several agents with different hyperparameters on the gambling task and measured the dynamics loss (\autoref{eq:dyn_loss}) and performance of each agent. Indeed, \textbf{agents with trajectory statistics that enable the dynamics model to ``cheat'' and get a low loss performed worse than those with representations independent from environment stochasticity}, empirically confirming \autoref{thm:main}.

%% file: sections/5_related.tex
\section{Related Work}

{\bf Offline Reinforcement Learning.} Offline reinforcement learning is a framework where logged data is used to learn behaviors \citep{levine_offline_2020}. In contrast to most online RL methods, offline RL avoids the need for often-expensive online data collection by using fixed datasets of trajectories generated by various policies and stitching together observed behaviors in order to optimize a reward function. The main challenge in the offline setting is that offline RL agents cannot collect more data to reduce their uncertainty and therefore offline RL approaches often have some type of value pessimism or policy constraints to keep the learned policy within the support of the offline dataset. Conservative Q-Learning \citep{kumar_conservative_2020} accomplishes this by learning a conservative lower-bound on the value of the current policy and achieve state-of-the-art performance on most offline-RL benchmarks. Other methods rely on ensembles (e.g., \citep{agarwal_optimistic_2020}) or policy regularization (e.g., \citep{wu_behavior_2019, kumar_stabilizing_2019, fujimoto_minimalist_2021}). Several standardized benchmarks have emerged for testing offline RL agents, including D4RL \citep{fu_d4rl_2020} and RL Unplugged \citep{gulcehre_rl_2020}. In contrast to our work, prior methods for offline RL primarily evaluate on deterministic environments, such as robotics and locomotion tasks as well as Atari tasks, which are near-deterministic \citep{kumar_conservative_2020}. 

{\bf Reinforcement Learning via Supervised Learning.} At the core of our work is the idea of reducing reinforcement learning to a \textit{prediction problem} and solving it via supervised learning. This idea was first proposed by \citet{schmidhuber_reinforcement_2019} and \citet{kumar_reward-conditioned_2019}, who proposed to learn behaviors by predicting actions conditioned on a desired outcome such as return. \citet{ghosh_learning_2021} proposed to use the same approach to solve goal-conditioned tasks and \citet{paster_planning_2021} proposed to predict multi-step action sequences, showing a connection between sequence modeling and optimal open-loop planning. Decision Transformer \citep{chen_decision_2021} proposed to use a transformer \citep{vaswani_attention_2017} to condition on trajectory history when predicting actions and tested the approach on offline RL rather than online, achieving results competitive with value-based offline RL approaches. Several followup extensions to Decision Transformer were explored, including using non-return goals \citep{furuta_generalized_2022} and using pretrained transformers \citep{reid_can_2022}. In a recent work, \citet{emmons_rvs_2021} coined the term RvS (reinforcement learning via supervised learning) to describe such approaches. Notably found that a transformer is not necessary to perform well on most tasks, showing that a two-layer MLP with sufficient regularization can actually outperform transformers at a fraction of the computational cost. 

No prior RvS approach has been thoroughly evaluated in stochastic environments. \citet{paster_planning_2021} describes a counter-example where an RvS agent in \citet{ghosh_learning_2021} doesn't converge in a stochastic environment, and proposes a solution. However, the approach can only be used to plan action sequences rather than reactive plans and therefore won't achieve maximal performance in many stochastic settings. \citet{ortega_shaking_2021} give a high level view of how sequence modeling can be affected by delusions in various problem settings when not treating actions as interventions. Our contribution is a more precise characterization of the problem within the framework of RvS which relates the choice of goal to environmental stochasticity and directly evokes an efficient algorithm for using RvS in stochastic environments without the need for explicit causal inference or intervention by carefully choosing the goals on which the agents conditions.

{\bf Adversarial Learning.} \algoshort{} uses an adversarial loss to ensure that trajectory statistics are independent from environment stochasticity. Adversarial losses are also widely used in fairness (e.g., \citep{beutel_data_2017, ganin_domain-adversarial_2017, zhang_mitigating_2018}), where it is used to prevent models from using sensitive attributes, and in generative modeling such as in GANs \citep{goodfellow_generative_2014}.

%% file: sections/6_discussion.tex
\section{Conclusion}

As models are scaled up in pursuit of better performance \citep{brown_language_2020}, it is important to acknowledge the limitations that cannot be fixed by more scale and data alone. Our work points out an issue with the increasingly popular Decision Transformer \citep{chen_decision_2021} and other RL via supervised learning algorithms when making decisions in stochastic environments. We show why such errors occur and show that these approaches work if and only if the goals that they are trained to achieve are independent from environment stochasticity. We give a practical algorithm that transforms problematic goals such as trajectory returns into expected returns that satisfy this condition. Finally, we validate our approach by testing it on several new stochastic benchmark tasks, showing empirically that our approach vastly outperforms return-conditioned models in terms of alignment between the target and average achieved return as well as maximum performance. We hope that our work gives insight into how predictive models can be used to make optimal decisions, even in stochastic environments, hopefully paving the way for the creation of more general and useful agents.

\paragraph{Limitations.}\label{par:limitations} There are several limitations and opportunities for future work. First, our approach relies on learning a dynamics model. In environments where a decently-strong dynamics model cannot be learned, the adversarial clustering in ESPER may fail. Second, ESPER relies on adversarial learning, which similarly to GANs \citep{goodfellow_generative_2014}, can have learning dynamics that are sensitive to hyperparameter choices (for example learning rate). Additionally, just like in k-means clustering, there are many possible cluster assignments for a given offline dataset. The final performance of ESPER can be affected by variance in the adversarial clustering phase. Finally, while we empirically validate our approach with return goals, we did not explore the use other types of goals where conditioning the model on the mean may be less appropriate, such as visual goals.

\paragraph{Societal Impact.}\label{par:impact} We believe that this work will result in a positive societal impact, since our approach can help to avoid unexpected behavior when controlling an agent trained using supervised learning. However, we acknowledge that powerful automated decision making algorithms such as \algoshort{} have the potential to make harmful decisions when trained with under-specified objectives or on biased data.

%% file: sections/7_appendix.tex
\section{Appendix}

\subsection{Proof for \autoref{thm:main}}
\label{sec:proof}

% \begin{proof}
% Let $\pi(a|s, z) = p_{\pi_{\mathcal{D}}}(a|s, I(\tau)=z)$. If $I(\tau) \perp s_{t+1}$ given $s_t, a_t$, then,

% \begin{align*}
%     \E_{\tau \sim p_{\pi_z}(\tau|s_0)}\left[D(I(\tau), z)\right] &\\
%     &\hspace{-8em}= \sum_{a_0, s_1, \ldots, s_T, a_T} \left(\prod_t p_{\pi_{\mathcal{D}}}(a_t|s_t, I(\tau)=z)\right) \left(\prod_t p(s_{t+1}|s_t, a_t)\right) D(I(\tau), z)\\
%     % &= \hspace{-2em}\sum_{a_0, s_1, \ldots, s_T, a_T} \left(\prod_t p_{\pi_{\mathcal{D}}}(a_t|s_t, I(\tau)=z)\right) \left(\prod_t p_{\pi_{\mathcal{D}}}(s_{t+1}|s_t, a_t)\right) D(I(\tau), z)\\
%     &\hspace{-8em}= \sum_{a_0, s_1, \ldots, s_T, a_T} \left(\prod_t p_{\pi_{\mathcal{D}}}(a_t|s_t, I(\tau)=z)\right) \left(\prod_t p_{\pi_{\mathcal{D}}}(s_{t+1}|s_t, a_t, I(\tau)=z)\right) D(I(\tau), z)\\
%     &\hspace{-8em}= \E_{\tau \sim p_{\pi_{\mathcal{D}}}(\tau|s_0, I(\tau) = z)}\left[D(I(\tau), z)\right]\\
%     &\hspace{-8em}= 0
% \end{align*}

% The second equality follows from the independence assumption and the final equality holds because the expectation is conditioned on $I(\tau) = z$.

% To show the other direction, assume $\E_{\tau \sim p_{\pi_z}(\tau|s_0)}\left[D(I(\tau), z)\right] = 0$. 

% Then, $p_{\pi_z}(\tau|s_0) > 0$ implies  $D(I(\tau), z) = 0$ since distance $D(\cdot, \cdot)$ is non-negative, which implies that when a trajectory has a non-zero probability, then $I(\tau) = z$. 

% Therefore, $p_{\pi_z}(\tau|s_0) = p_{\pi_z}(\tau|s_0, I(\tau)=z)$, which finally implies that $p_{\pi_z}(s_{t+1}|s_t, a_t) = p_{\pi_z}(s_{t+1}|s_t, a_t, I(\tau)=z)$. 

% \end{proof}

\begin{proof}
Let $\pi(a_t|s_0, a_0, \ldots, s_t, z) = p_{\pi_{\mathcal{D}}}(a_t|s_0, a_0, \ldots, s_t, I(\tau) = z)$ and

let $p_{\pi_{\mathcal{D}}}(s_t|s_0, a_0, \ldots, a_{t-1}) = p_{\pi_{\mathcal{D}}}(s_t|s_0, a_0, \ldots, a_{t-1}, I(\tau) = z)$. Then,

\begin{align*}
    \E_{\tau \sim p_{\pi_z}(\tau|s_0)}\left[D(I(\tau), z)\right] &\\
    &\hspace{-10em}= \sum_{a_0, s_1, \ldots, s_T, a_T} p_{\pi_z}(\tau|s_0) D(I(\tau), z)\\
    &\hspace{-10em}= \sum_{a_0, s_1, \ldots, s_T, a_T} \prod_t \pi(a_t|s_0, a_0, \ldots, s_t, z)p(s_{t+1}|s_0, a_0, \ldots, a_t) D(I(\tau), z)\\
    &\hspace{-10em}= \sum_{a_0, s_1, \ldots, s_T, a_T} \prod_t p_{\pi_{\mathcal{D}}}(a_t|s_0, a_0, \ldots, s_t, I(\tau) = z)p_{\pi_{\mathcal{D}}}(s_t|s_0, a_0, \ldots, a_{t-1}, I(\tau) = z)D(I(\tau), z)\\
    &\hspace{-10em}= \sum_{a_0, s_1, \ldots, s_T, a_T} p_{\pi_{\mathcal{D}}}(\tau|s_0, I(\tau)=z) D(I(\tau), z)\\
    &\hspace{-10em}= \E_{\tau \sim p_{\pi_{\mathcal{D}}}(\tau|s_0, I(\tau) = z)}\left[D(I(\tau), z)\right]\\
    &\hspace{-10em}= 0
\end{align*}
\end{proof}

The first equality follows from the definition of expectation. The second equality follows by the definition of a trajectory (the joint probability over the sequence of states and actions) and by the chain rule. The third equality follows due to our assumptions. Finally, the last two equalities use the definition of a trajectory and expectation.

To show the other direction, assume $\E_{\tau \sim p_{\pi_z}(\tau|s_0)}\left[D(I(\tau), z)\right] = 0$.

Then, $p_{\pi_z}(\tau|s_0) > 0$ implies  $D(I(\tau), z) = 0$ since distance $D(\cdot, \cdot)$ is non-negative, which implies that when a trajectory has a non-zero probability, then $I(\tau) = z$. 

Therefore, $p_{\pi_z}(\tau|s_0) = p_{\pi_z}(\tau|s_0, I(\tau)=z)$, which finally implies that $p_{\pi_{\mathcal{D}}}(s_t|s_0, a_0, \ldots, a_{t-1}) = p_{\pi_{\mathcal{D}}}(s_t|s_0, a_0, \ldots, a_{t-1}, I(\tau) = z)$.

\subsubsection{Regarding History Conditioning}

We thank the authors of \citet{yang_dichotomy_2022}, who pointed out an error in the theory in a previous draft of this paper via a counter-example in Appendix C. Our previous proof incorrectly assumed that independence from past states and actions, via our choice of Markov policies and via the Markov assumptions in MDPs, held when conditioning on $I(\tau)$. This is not in general correct, and we have updated our theory to reflect that theoretical guarantees under our framework require a memory-based policy and for the conditioning trajectory statistics to be independent of the future state given the entire history of the trajectory, not just the most recent state and action.

We note that in practice, Markov policies work well with ESPER, as does optimizing trajectory statistics for independence given only the most recent state.

\subsection{\algoshort{} Generalizes Return-Conditioning}

While ESPER is designed to fix the problem with RvS in stochastic environments, it still performs well in deterministic environments. In a deterministic environment when the dynamics model can already predict the next state perfectly with just the current state and action, only the policy reconstruction loss will be active. Therefore, the optimal clustering will have each trajectory in its own cluster and labeled with their original return, and ESPER reduces to return-conditioned RvS. In \autoref{fig:mujoco} we show empirically that the performance of ESPER matches return-conditioned Decision Transformer in the deterministic D4RL Mujoco tasks \citep{fu_d4rl_2020}.

% \begin{table}[!ht]
%     \centering
%     \begin{tabular}{lllrrr}
%     \toprule
%         \textbf{Dataset} & \textbf{Environment} & \textbf{Target} & \textbf{ESPER (ours)} & \textbf{DT} & \textbf{CQL} \\ \midrule
%         Medium-Expert & hopper & 3600 & 89.95$\pm$13.91 & 79.64$\pm$34.45 & \textbf{110.0} \\
%         Medium-Expert & walker & 5000 & \textbf{106.87$\pm$1.26} & \textbf{107.96$\pm$0.63} & 98.7 \\
%         Medium-Expert & half-cheetah & 6000 & \textbf{66.95$\pm$11.13} & 42.89$\pm$0.35 & \textbf{62.4} \\ \midrule
%         Medium & hopper & 3600 & 50.57$\pm$3.43 & \textbf{59.46$\pm$4.74} & \textbf{58.0} \\ 
%         Medium & walker & 5000 & 69.78$\pm$1.91 & 69.7$\pm$7.12 & \textbf{79.2}\\ 
%         Medium & half-cheetah & 6000 & 42.31$\pm$0.08 & 42.32$\pm$0.39 & \textbf{44.4} \\\midrule 
%         Medium-Replay & hopper & 3600 & \textbf{50.20$\pm$16.09} & \textbf{61.94$\pm$16.99} & 48.6 \\ 
%         Medium-Replay & walker & 5000 & \textbf{65.48$\pm$8.05} & \textbf{63.77$\pm$2.82} & 26.7 \\ 
%         Medium-Replay & half-cheetah & 6000 & 35.85$\pm$1.97 & 36.88$\pm$0.36 & \textbf{46.2} \\ \bottomrule
%     \end{tabular}
%     \caption{Results on D4RL Mujoco Tasks\vspace{-1em}}
%         \label{fig:mujoco}
% \end{table}

\begin{figure}[!ht]
    \centering
    \begin{center}
\begin{tabular}{lllrrr}
    \toprule
        \textbf{Dataset} & \textbf{Environment} & \textbf{Target} & \textbf{ESPER (ours)} & \textbf{DT} & \textbf{CQL} \\ \midrule
        Medium-Expert & hopper & 3600 & 89.95$\pm$13.91 & 79.64$\pm$34.45 & \textbf{110.0} \\
        Medium-Expert & walker & 5000 & \textbf{106.87$\pm$1.26} & \textbf{107.96$\pm$0.63} & 98.7 \\
        Medium-Expert & half-cheetah & 6000 & \textbf{66.95$\pm$11.13} & 42.89$\pm$0.35 & \textbf{62.4} \\ \midrule
        Medium & hopper & 3600 & 50.57$\pm$3.43 & \textbf{59.46$\pm$4.74} & \textbf{58.0} \\ 
        Medium & walker & 5000 & 69.78$\pm$1.91 & 69.7$\pm$7.12 & \textbf{79.2}\\ 
        Medium & half-cheetah & 6000 & 42.31$\pm$0.08 & 42.32$\pm$0.39 & \textbf{44.4} \\\midrule 
        Medium-Replay & hopper & 3600 & \textbf{50.20$\pm$16.09} & \textbf{61.94$\pm$16.99} & 48.6 \\ 
        Medium-Replay & walker & 5000 & \textbf{65.48$\pm$8.05} & \textbf{63.77$\pm$2.82} & 26.7 \\ 
        Medium-Replay & half-cheetah & 6000 & 35.85$\pm$1.97 & 36.88$\pm$0.36 & \textbf{46.2} \\ \bottomrule
    \end{tabular}
\end{center}
    \caption{Results on D4RL Mujoco Tasks\vspace{-1em}}
        \label{fig:mujoco}
\end{figure}

\subsection{\algoshort{} with MLPs}
\label{sec:mlp}

As reported in \citet{emmons_rvs_2021}, a transformer is not necessary to get strong performance with RvS on many environments. We also find that we get similar results when training using a simple MLP with three hidden layers, as shown in \autoref{fig:mlp_table}.

\begin{figure}[t]
    \centering
    \begin{center}
\begin{tabular}{llll} 
 \toprule
 Task & Return-Conditioned RvS (MLP) & CQL & \algoshort{} (MLP) (Ours) \\
 \midrule
 Gambling & -0.05 (0.27) & \textbf{1.0 (0.0)} & \textbf{1.0 (0.0)}\\ 
 Connect Four & 0.24 (0.15) & 0.61 (0.05) & \textbf{0.99 (0.01)} \\ 
 2048 & 0.56 (0.03) & 0.7 (0.09) & \textbf{0.81 (0.05)} \\
%  Half Cheetah (G) & \textbf{491.28 (32.4)} & \textbf{475.57 (21.85)} & 0 \\
 \bottomrule
\end{tabular}
\end{center}
    \caption{\algoshort{} performs similarly on benchmark tasks when using a simple MLP rather than a transformer for the policy.\vspace{-1em}}
    \label{fig:mlp_table}
\end{figure}

\subsection{Benchmark Task Details}
\label{sec:benchmark_details}

\subsubsection{An Illustrative Gambling Environment}

To clearly illustrate the issue with prior approaches when conditioning on outcome variables that are correlated with environment stochasticity, we run our approach on a simple gambling environment. This environment, illustrated in \autoref{fig:toy}, has three actions: one which will result in the agent gaining one reward, and two gambling actions where the agent could receive either positive or negative reward.

\subsubsection{Multi-Agent Game: Connect Four}

Connect Four is a popular two-player board game where players alternate in placing tiles in the hope to be the first to get four in a row. In this task, we consider a single agent version of Connect Four where the opponent is fixed to be a stochastic agent. This is a realistic setting, since in the real world, the single greatest source of stochasticity will likely be other agents that the agent is interacting with. An ideal agent will need to take this into account to make optimal decisions.

Since Connect Four can be optimally solved with search techniques, we set the opposing agent to be optimal\footnote{We use the solver from \href{https://github.com/PascalPons/connect4}{https://github.com/PascalPons/connect4}.}, with a small chance that it won't place a piece in the rightmost column when it is optimal to do so. This creates an MDP with two ways to win: first, the agent can play optimally to the end of the game to guarantee a win (the first player always can win); second, to win quickly, the agent can place four pieces on the rightmost column with the hope that the opponent will not block (which happens with a low probability). 

The offline dataset for this task is generated by using an epsilon-optimal agent with 50\% probability and an exploiter agent that only places pieces on the right with 50\% probability.

\subsubsection{Stochastic Planning in 2048}

2048 \citep{cirulli_2048_2014} is a single player puzzle game where identical tiles are combined in order to build up tiles representing different powers of two. With each move, a new tile randomly appears on the board, and the game ends when no moves are available. A strong 2048 agent will consider the different possible places new tiles will appear in order to maximize the potential for combining tiles in the future.

Since the vanilla version of 2048 can require billions of steps to solve with reinforcement learning \citep{antonoglou_planning_2021}, we modified the game\footnote{We used the implementation of 2048 found at \href{https://github.com/activatedgeek/gym-2048}{https://github.com/activatedgeek/gym-2048}.} by terminating the episode when a 128 tile is created. The agent gets one reward for successfully combining tiles in order to reach 128 and zero reward otherwise.

The offline dataset for this task is generated using a mixture of trajectories from an agent trained with PPO \citep{schulman_proximal_2017} using the implementation in Stable Baselines 3 \citep{raffin_stable-baselines3_2021} and a random policy.

% \subsection{Stochasticity-Independent Representations Perform Better}

% In order to verify that the degree to which the learned trajectory representations (cluster assignments) are independent of environment stochasticity affects the empirical performance of an agent, we set out to empirically measure this. We did so by training several agents in the gambling environment with different hyperparameters (specifically $\beta_{\text{adv}}$) and measured their performance. Results are shown in \autoref{fig:adv_plot}. Clustering assignments that enabled the dynamics model to predict the next state more accurately than should be possible lead to agents that performed worse.

% \label{sec:adv}

% \begin{figure}[t]
%     \centering
%     \includegraphics[width=0.5\textwidth]{figures/dyn.pdf}
%     \caption{We trained several agents with different hyperparameters on the gambling task. Notably, agents with trajectory statistics that enabled the dynamics model to ``cheat'' and get a low loss performed worse than those with representations independent from environment stochasticity, empirically confirming \autoref{thm:main}.}
%     \label{fig:adv_plot}
% \end{figure}

\subsection{Training Details}
\label{sec:training_details}

\subsubsection{Hyperparameters}

The most important hyperparameters to tune for our method are the tradeoff between reconstructing actions and removing dynamics information from the clusters, controlled by $\beta_{\text{act}}$, and the number of clusters, controlled by \verb|rep_size| and \verb|rep_groups|. The trajectory representation (i.e. cluster assignment) is formed by sampling from \verb|rep_groups| categorical variables of dimension \verb|rep_size| / \verb|rep_groups|. These values were tuned per environment using a simple grid search. Specific hyperparameter values for each environment can be found at \autoref{tab:hyp}.

\begin{table}[!h]
    \centering
    \begin{tabular}{|c|c|}\hline
    \multicolumn{2}{|c|}{Decision Transformer}\\\hline
    \verb|batch_size| & 64 \\\hline
    \verb|learning_rate| & 5e-4 \\ \hline
    \verb|policy_architecture| & Transformer \\ \hline
    \verb|embed_dim| & 128\\ \hline
    \verb|n_layer| & 3\\ \hline
    \verb|n_head| & 1\\ \hline
    \verb|activation_fn| & ReLU\\ \hline
    \verb|dropout| & 0.1\\ \hline
    \verb|n_head| & 1\\ \hline
    \verb|weight_decay| & 1e-4\\ \hline
    \verb|warmup_steps| & 10000\\ \hline
    \verb|discount| $\gamma$ & 1 \\ \hline
    \verb|eval_samples| & 100 \\ \hline
    \hline \multicolumn{2}{|c|}{ESPER} \\ \hline
    \verb|learning_rate| & 1e-4 \\ \hline
    \verb|batch_size|   & 100 \\\hline
    \verb|hidden_size| & 512 \\ \hline
    \verb|policy.hidden_layers| & 3 \\ \hline
    \verb|clustering_model.hidden_layers| & 2 \\\hline
    \verb|clustering_model.lstm_hidden_size| & 512 \\\hline
    \verb|clustering_model.lstm_layers| & 1 \\\hline
    \verb|action_predictor.hidden_layers| & 2 \\\hline
    \verb|return_predictor.hidden_layers| & 2 \\\hline
    \verb|transition_predictor.hidden_layers| & 2 \\\hline
    \verb|activation_fn| & ReLU \\\hline
    \verb|optimizer| & AdamW \citep{loshchilov_decoupled_2019} \\\hline
    \verb|normalization| & batch norm \citep{ioffe_batch_2015} \\ \hline
    \end{tabular}
    \vspace{1em}
    \caption{\algoshort{} hyperparameters}
    \label{tab:hyp}
\end{table}

\begin{table}[!h]
    \centering
    \begin{tabular}{|c|c|}
    \hline \multicolumn{2}{|c|}{Gambling} \\ \hline
    \verb|rep_size| & 8 \\ \hline
    \verb|rep_groups| & 1 \\ \hline
    $\beta_{\text{act}}$ & 0.01 \\ \hline
    $\beta_{\text{adv}}$ & 1 \\ \hline
    \verb|cluster_epochs| & 5 \\ \hline
    \verb|label_epochs| & 1 \\ \hline
    \verb|policy_steps| & 50000 \\ \hline
    \hline \multicolumn{2}{|c|}{Connect Four} \\ \hline
    \verb|rep_size| & 128 \\ \hline
    \verb|rep_groups| & 4 \\ \hline
    $\beta_{\text{act}}$ & 0.05 \\ \hline
    $\beta_{\text{adv}}$ & 1 \\ \hline
    \verb|cluster_epochs| & 5 \\ \hline
    \verb|label_epochs| & 5 \\ \hline
    \verb|policy_steps| & 50000 \\ \hline
    \hline \multicolumn{2}{|c|}{2048} \\ \hline
    \verb|rep_size| & 128 \\ \hline
    \verb|rep_groups| & 4 \\ \hline
    $\beta_{\text{act}}$ & 0.02 \\ \hline
    $\beta_{\text{adv}}$ & 1 \\ \hline
    \verb|cluster_epochs| & 4 \\ \hline
    \verb|label_epochs| & 1 \\ \hline
    \verb|policy_steps| & 50000 \\ \hline
    \hline \multicolumn{2}{|c|}{Mujoco} \\ \hline
    \verb|rep_size| & 256 \\ \hline
    \verb|rep_groups| & 4 \\ \hline
    $\beta_{\text{act}}$ & 0.03 \\ \hline
    $\beta_{\text{adv}}$ & 1 \\ \hline
    \verb|cluster_epochs| & 5 \\ \hline
    \verb|label_epochs| & 5 \\ \hline
    \verb|policy_steps| & 100000 \\ \hline
    \end{tabular}
    \vspace{1em}
    \caption{Environment hyperparameters}
    \label{tab:env_hyp}
\end{table}

\subsubsection{Computation}

Our experiments were run on T4 GPUs and running our algorithm took only around $1$ hour per seed. We used PyTorch \citep{paszke_pytorch_2019} and experiments were tracked using Weights and Biases \citep{biewald_experiment_2020}.

\subsection{Baselines}

We used the codebase provided by the authors of Decision Transformer \citep{chen_decision_2021} to implement our experiments. Other than hyperparameters specific to ESPER, Decision Transformer and ESPER use the same hyperparameters found in \autoref{tab:hyp}.For Conservative Q-Learning (CQL) \citep{kumar_conservative_2020}, we used the default implementation from d3rlpy, an offline deep RL library \citep{seno_d3rlpy_2021}.

\begin{algorithm2e}
\caption{\algoshort{}}\label{algo:esper}
\KwData{Dataset $\mathcal{D}$ consisting of trajectories of states, actions, and rewards}
\For{\text{cluster iteration} $k = 1, 2, \ldots$}{
    $s, a \gets \text{sample batch of trajectories from } \mathcal{D}$\;
    $\operatorname{assignments} \gets \operatorname{ClusterAssignments}(s, a)$\;
    $\hat{a} \gets \operatorname{ActionPredictor}(s, \operatorname{assignments})$\;
    $\hat{s}_{t+1} \gets \operatorname{TransitionPredictor}(s, a, \operatorname{assignments})$\;
    Update cluster assignments by training $\hat{a}$ to predict $a$ and $\hat{s}_{t+1}$ to \textit{not} predict $s_{t+1}$ by minimizing \autoref{eq:rep_loss}\;
    Train the $\operatorname{TransitionPredictor}$ to predict next states by minimizing \autoref{eq:dyn_loss}\;
}
Fit a model $f_{\psi}(I(\tau))$ to predict the average trajectory return $\hat{R}$ in each cluster\;
Create dataset $\mathcal{D'}$ of states, actions, and average returns\;
\For{\text{policy iteration} $k = 1, 2, \ldots$}{
    $s, a, \hat{R} \gets \text{sample batch of states, actions, and average returns from } \mathcal{D}'$\;
    Train the policy $\pi(a|s, \hat{R})$ by minimizing \autoref{eq:policy_loss}\;
}
\end{algorithm2e}

\subsection{\algoshort{} Pseudocode}
\label{sec:code}

We provide pseudocode for \algoshort{} in \autoref{algo:esper} and detailed pseudocode (roughly following the syntax used in PyTorch \citep{paszke_pytorch_2019}) for the clustering step of \algoshort{} in \autoref{algo:cluster}.

\begin{algorithm2e}[hb]
% \SetAlgoLined
\definecolor{codeblue}{rgb}{0.28125,0.46875,0.8125}
\lstset{
    basicstyle=\fontsize{9pt}{9pt}\ttfamily\bfseries,
    commentstyle=\fontsize{9pt}{9pt}\color{codeblue},
    keywordstyle=
}
% Thanks decision transformer authors for the code style
\begin{lstlisting}[language=python,frame=none]
# s, a, seq_len: states, actions, sequence lengths (in case of early termination)
# Models:
# encoder_mlp - MLP
# temporal_encoder - LSTM
# rep_mlp - MLP
# act_mlp - MLP
# dynamics_mlp - MLP

def ClusterAssignments(s, a):
    bsz, t = s.shape[:2]
    x = torch.cat((s, a), dim=-1).view(bsz, t, -1)
    x = torch.flip(x, dims=[1])
    x = torch.encoder_mlp(x)
    x, hidden = temporal_encoder(x, init_hidden())
    x = torch.flip(x, dims=[1])
    x = rep_mlp(x)
    cluster_assignments = F.gumbel_softmax(x)
    return cluster_assignments
    
def ActionPredictor(s, cluster_assignments):
    # cluster_assignments.shape = [bsz, t, -1]
    # For a timestep t, we sample a cluster assignment for a timestep
    # from 0, ..., t randomly
    past_assignments = sample_past_assignments(cluster_assignments)
    x = torch.cat((s, past_assignments), dim=-1)
    pred_next_action = act_mlp(x)
    return pred_next_action
    
def TransitionPredictor(s, a, cluster_assignments):
    # cluster_assignments.shape = [bsz, t, -1]
    # For a timestep t, we sample a cluster assignment for a timestep
    # from 0, ..., t randomly
    past_assignments = sample_past_assignments(cluster_assignments)
    x = torch.cat((s, a, past_assignments), dim=-1)
    pred_next_s = dynamics_mlp(x)
    return pred_next_s
    
# training loop    
for (s, a) in dataloader:
    # get cluster assignments for the trajectories
    cluster_assignments = ClusterAssignments(s, a) 
    # predict actions based on clusters
    pred_next_action = ActionPredictor(s, cluster_assignments) 
    # predict state transitions based on clusters
    pred_next_s = TransitionPredictor(s, a, cluster_assignments) 
    # optimize the clusters for action prediction and to hurt next state prediction
    cluster_loss = act_loss(pred_next_action, a) - state_loss(pred_next_s, s) 
    cluster_loss.zero_grad(); cluster_loss.backward(); cluster_optimizer.step();
    # optimize the transition predictor
    dyn_loss = state_loss(pred_next_s, s)
    dyn_loss.zero_grad(); dyn_loss.backward(); dyn_optimizer.step();

\end{lstlisting}
\caption{\algoshort{} - Adversarial Clustering}
\label{algo:cluster}
\end{algorithm2e}